\documentclass[runningheads]{llncs}
\usepackage[T1]{fontenc}
\usepackage{graphicx}
\usepackage{booktabs}
\usepackage[misc]{ifsym}

\usepackage{mwe}
\usepackage{cite}
\usepackage{amsmath,amssymb,amsfonts}
\usepackage{algorithmic}
\usepackage{graphicx}
\usepackage{textcomp}
\usepackage{xcolor}

\usepackage{booktabs}
\usepackage{algorithm}
\usepackage{algorithmic}
\usepackage{subfigure}
\usepackage{multirow}
\usepackage{textcomp,moresize,wrapfig}
\usepackage{epstopdf}

\usepackage[utf8]{inputenc} % allow utf-8 input
\usepackage[T1]{fontenc}    % use 8-bit T1 fonts
\usepackage{hyperref}       % hyperlinks
\usepackage{url}            % simple URL typesetting
\usepackage{booktabs}       % professional-quality tables
\usepackage{amsfonts}       % blackboard math symbols
\usepackage{nicefrac}       % compact symbols for 1/2, etc.
\usepackage{microtype}      % microtypography
\usepackage{xcolor}         % colors

\usepackage{helvet} % DO NOT CHANGE THIS
\usepackage{courier}  % DO NOT CHANGE THIS
\usepackage{stmaryrd}
\usepackage{pifont}% http://ctan.org/pkg/pifont

\usepackage{times}
\usepackage{soul}
\usepackage{url}
\usepackage[small]{caption}
\usepackage{graphicx}
\usepackage{amsmath}
\usepackage{booktabs}
\usepackage{algorithm}
\usepackage{algorithmic}
\usepackage[switch]{lineno}

\usepackage{subfigure}
\usepackage{multirow}
\usepackage{textcomp,wrapfig}
\usepackage{hhline}
\usepackage{amssymb}
\usepackage{epstopdf}
\usepackage{color}
\usepackage{enumitem,picinpar}
\usepackage{newfloat}
\usepackage{listings}
\usepackage{textcomp,moresize}
\usepackage{bbm}
\usepackage{latexsym}
\usepackage{epsfig}
\usepackage[flushleft]{threeparttable}
\usepackage{amsfonts}

\renewcommand{\algorithmicrequire}{\textbf{Input:}}
\renewcommand{\algorithmicensure}{\textbf{Output:}}

\DeclareMathOperator*{\argmin}{arg\,min}

\newcommand{\iI}[1]{\llbracket {#1}\rrbracket_{\mathbb{I}}}
\newcommand{\XI}[1]{\llbracket {#1}\rrbracket_{X,I}}

\newcommand{\cmark}{\ding{51}}%
\newcommand{\xmark}{\ding{55}}%

\newcommand{\A}{\mathcal{A}}

\begin{document}

\title{Safe Screening Rules for Group SLOPE}

\titlerunning{Safe Screening Rules for Group SLOPE}
% If the full title of your paper is short enough to also fit in the running head, you can omit the abbreviated paper title here. You can check as follows: if you comment out the \titlerunning line, something will appear in the header of all odd-numbered pages of your PDF from page 3 onward. This something is either the full title (in which case all is well), or the error message "Title Suppressed Due to Excessive Length". If this error message appears, you're going to want to provide an abbreviated title within the \titlerunning command, because if you won't do it, Springer will do it for you.

%N.B.: Author information (both in the \author{} and \authorrunning{} command) should only be present in the Camera-Ready Version of your paper. The version that you initially submit for review, ought to be double-blind. So, when initially submitting your paper, use:
%\author{Author information scrubbed for double-blind reviewing}
\author{Runxue Bao \inst{1}(\Letter) \and
Quanchao Lu\inst{2}  \and
Yanfu Zhang\inst{3}}
% You may leave out the orcidID information, if you want to.
% Use \corr to indicate the corresponding author. Note the spacing around the \corr command. Only one author can be the corresponding author.

%N.B.: comment out the \authorrunning{} command for the double-blind version of your paper submitted for review. Later, if your paper is accepted, use the command for the Camera-Ready Version.
%\authorrunning{A.L. Benjamin et al.}
% If there are two authors, write 'A.L. Benjamin and C.C. Broadus Jr.'
% If there are more than two authors, '[...] et al.' is used.

\institute{University of Pittsburgh, Pittsburgh, PA 15260, United States \\ \email{runxue.bao@pitt.edu}
\and
Georgia Institute of Technology, Atlanta, GA 30332, United States \\ \email{qlu43@gatech.edu}
\and
The College of William and Mary, Williamsburg, VA 23185, United States \\
\email{yzhang105@wm.edu}}

\tocauthor{Runxue Bao, Quanchao Lu, Yanfu Zhang}
\toctitle{Safe Screening Rules for Group SLOPE}

\maketitle              % typeset the header of the contribution

\begin{abstract}
Variable selection is a challenging problem in high-dimensional sparse learning, especially when group structures exist. Group SLOPE performs well for the adaptive selection of groups of predictors. However, the block non-separable group effects in Group SLOPE make existing methods either invalid or inefficient. Consequently, Group SLOPE tends to incur significant computational costs and memory usage in practical high-dimensional scenarios. To overcome this issue, we introduce a safe screening rule tailored for the Group SLOPE model, which efficiently identifies inactive groups with zero coefficients by addressing the block non-separable group effects. By excluding these inactive groups during training, we achieve considerable gains in computational efficiency and memory usage. Importantly, the proposed screening rule can be seamlessly integrated into existing solvers for both batch and stochastic algorithms. Theoretically, we establish that our screening rule can be safely employed with existing optimization algorithms, ensuring the same results as the original approaches. Experimental results confirm that our method effectively detects inactive feature groups and significantly boosts computational efficiency without compromising accuracy.

\end{abstract}

\begin{table*}[t]
	\center
	\caption{Representative safe screening algorithms.}
	\begin{tabular}{c|c|c|c|c|c}
		\hline
		{\textbf{Problem}} & {\textbf{Reference}}& {\textbf{Group-wise}}  & {\textbf{Inseparability}} & {\textbf{Group Effects}} & {\textbf{Dynamic}}  \\
		\hline
Lasso & \cite{fercoq2015mind} & \xmark  & \xmark & \xmark &  Singly  \\
Logistic Regression & \cite{wang2014safe} & \xmark& \xmark & \xmark &  \xmark  \\
Proximal Weighted Lasso & \cite{rakotomamonjy2019screening} & \xmark & \xmark & \xmark & Singly \\
SLOPE & \cite{larsson2020strong} & \xmark  & \cmark  & \xmark &  Singly  \\
Group Lasso  & \cite{bonnefoy2015dynamic} & \cmark   & \xmark & \xmark & Singly \\
Sparse Group Lasso & \cite{wang2014two} &\cmark   & \xmark & \xmark & Singly  \\
Tree Structured Group Lasso & \cite{wang2015multi} & \cmark & \xmark & \xmark &  \xmark   \\
Sparse-group Lasso &\cite{ndiaye2016gap} & \cmark & \xmark & \xmark & Singly
		\\
		\hline
Group SLOPE & Ours &  \cmark  & \cmark & \cmark& Doubly \\
		\hline
	\end{tabular}
	\label{table:methods}
\end{table*}

\section{Introduction}
Group structures are ubiquitous in many high-dimensional problems with massive correlated and superfluous features. To obtain more stable and interpretable models with better prediction performance, many sparse learning models with grouping structures were proposed and achieved great success in many real-world applications. Group Lasso \cite{yuan2006model} and its variants, including Sparse Group Lasso \cite{simon2013sparse}, composite absolute penalties \cite{zhao2009composite}, tree Lasso \cite{kim2010tree}, and Overlapping Group Lasso \cite{jacob2009group,jenatton2011structured}, are the most popular ones for group feature selection, which encourages structured sparsity with the prior information of feature group structures.

In this paper, we focus on the adaptive group feature selection model - Group SLOPE \cite{gossmann2015identification,gossmann2017sparse,brzyski2019group}. Let design matrix $X = [x_{1},\ldots,x_{n}]^\top \in \mathbb{R}^{n \times d}$ have $n$ observations and $d$ variables and $y \in \mathbb{R}^{n}$ denote the  measurement vector. Given $I$ is a partition of the set $\{1,\ldots,d\}$ and $W$ is a diagonal matrix with $W_{i,i}:=w_i$ for $i=1,\ldots,m$, $X_{I_i} \in \mathbb{R}^{n \times |I_i|}$ denotes a partition of the matrix $X$ and $\XI{\beta}:= \big(\|X_{I_1} \beta_{I_1}\|_2,  \ldots, \|X_{I_m} \beta_{I_m}\|_2\big)^\mathsf{T}$ denotes the group effects,  Group SLOPE can be formulated as follows:
\begin{eqnarray} \label{groupslope}
    \min\limits_{\beta}  P_{\lambda}(\beta):=\frac{1}{2}\|y-X\beta\|^2_{2}  + J_{\lambda}(W\XI{\beta}),
\end{eqnarray}
 where $\beta \in \mathbb{R}^{d}$ denotes the unknown coefficient vector and $\lambda = [\lambda_{1},  \ldots, \lambda_{m}]$ is a non-negative regularization parameter vector of $m$ non-increasing weights that $\lambda_1 \geq  \ldots \geq \lambda_m$. The term $J_{\lambda}(b)$ denotes the ordered weighted $l_1$-norm as $J_{\lambda}(b) = \sum_{i=1}^m\lambda_{i}|b|_{[i]}$ where $b_{[1]} \geq \ldots \geq b_{[m]} $ are the ordered terms. 

Group SLOPE penalty $J_{\lambda}(W\XI{b})$ adaptively penalizes the group effects based on the magnitude. Thus, Group SLOPE can simultaneously encourage the group effects to be equal and sparse, which is helpful to denoise and improve the prediction. \cite{brzyski2019group} provided both nice empirical results and theoretical analysis for Group SLOPE on the adaptive selection of groups of predictors. In general, Group SLOPE can achieve the exact minimax estimation without any knowledge of coefficients sparsity and control the group false discovery rate at a specific level \cite{gossmann2017sparse,brzyski2019group}. The attractive properties above, which do not simultaneously exist in other models such as Group Lasso and SLOPE \cite{bogdan2015slope}, had made Group SLOPE an effective method for the analysis in the high-dimensional setting \cite{gossmann2015identification,gossmann2017sparse,brzyski2019group}. Please note that Group SLOPE includes a broad set of sparse learning models. For example, Group SLOPE reduces to Group Lasso when $\lambda_{1} =  \ldots = \lambda_{m}$ and $w_i = \sqrt{|I_i|}$. Group SLOPE reduces to SLOPE when each group only has one variable and $X$ is standardized to have unit column norms. Besides, Group SLOPE certainly includes Lasso, weighted Lasso \cite{bergersen2011weighted} and $L_{\infty}$-norm regression. Also, Group SLOPE can be easily extended to the logistic loss for classification tasks.

From an optimization perspective, the block-nonseparable group effects render the coordinate descent algorithm for Group Lasso ineffective. To address the computational challenges of Group SLOPE, proximal gradient methods have been introduced \cite{brzyski2019group}. However, these methods encounter significant computational and memory challenges, particularly in high-dimensional settings. This is primarily because the algorithm processes all data points at each iteration, even when group coefficients are zero. 

Various screening rules have been developed to speed up the training of sparse learning models by eliminating inactive features. \cite{Laurent2012safe} introduced a static safe screening rule for $l_{1}$-regularized problems, which eliminates features before the optimization process. By relaxing the requirements of the safe rule, \cite{tibshirani2012strong} developed a strong rule for Lasso, utilizing heuristic strategies within an active set approach. Nevertheless, this method can potentially exclude important features, making it necessary to perform a supplementary KKT condition check. An analogous strong screening rule for Group SLOPE has also been proposed in \cite{feser2024strong}. Additionally, the sequential safe rule presented in \cite{wang2013lasso,xiang2016screening} relies on the exact dual optimal solution, making it potentially time-consuming. More recently, \cite{fercoq2015mind} proposed a dynamic screening rule for Lasso, which is applied throughout the learning process, based on the duality gap, offering provable safety and improved speed compared to previous rules. This has led to the development of many dynamic screening rules for various sparse learning models \cite{shibagaki2016simultaneous, ndiaye2016gap, rakotomamonjy2019screening, bao2020fast, bao2022accelerated, bao2022doubly, bao2025safe}, all aimed at enhancing training efficiency. Thus, improving the efficiency of solving Group SLOPE via dynamic screening rules becomes both
important and promising. Moreover, by reducing the number of model parameters, such techniques can also enhance inference performance, similar to the benefits observed with pruning methods \cite{han2015learning,frankle2018lottery,wu2024auto,lu2024all}.

The goal of this research is to expedite the training process of Group SLOPE models through the application of safe screening techniques, enabling the secure exclusion of inactive groups whose parameters are guaranteed to be zeros during training. Table \ref{table:methods} outlines various representative safe screening algorithms, highlighting that such rules have been developed to boost the efficiency of training algorithms across numerous sparse learning models. However, the complex penalty structure of Group SLOPE, characterized by its block-nonseparable nature in relation to group effects, has so far hindered the development of safe screening rules for this model. The challenges can be summarized as follows: Firstly, unlike other models that penalize coefficients directly, Group SLOPE penalizes group effects $\|X\beta\|_{2}$, which does not directly enforce coefficient sparsity. Secondly, while other models are restricted to either feature-wise or separable group-wise penalties, Group SLOPE introduces the first non-separable group-wise feature selection method, with all hyperparameters for each group remaining unfixed during the training process—unlike in models such as Group Lasso or Sparse Group Lasso, where hyperparameters are predetermined before optimization. 

In response to these challenges, this paper introduces a doubly dynamic safe screening rule tailored to the general Group SLOPE models. This represents, to our knowledge, the first safe screening rule specifically designed for adaptive group feature selection models. In high-dimensional settings where many groups have zero-valued coefficients, our screening rule efficiently identifies and excludes these inactive groups, thereby accelerating the original algorithms. Our approach begins by decoupling the design matrix to manage the block non-separable group effects. We then establish a doubly dynamic screening rule featuring a decreasing left bound and an increasing right bound, resulting in an expanding safe region. Crucially, the proposed screening rule is solver-independent and can be seamlessly integrated into existing iterative algorithms. Empirical evaluations on four benchmark datasets confirm that our approach yields significant computational advantages.

\section{Safe Screening Rules for Group SLOPE}
In this section, we first decouple the over-complex group effect penalty and then propose a doubly dynamic safe screening rule for Group SLOPE.

\subsection{Decoupling the Group Effect Penalty}
Different from other models, Group SLOPE penalizes group effects $\|X\beta\|_{2}$ directly. To propose a screening rule for Group SLOPE, we first derive an equivalent formulation of (\ref{groupslope}), which decomposes the design matrix as an orthogonal matrix and a corresponding full-row rank matrix. Specifically, by representing $X_{I_i}=U_iR_i$ where $U_i$ is any matrix with $|I_i|$ orthogonal columns and $R_i$ is the corresponding full-row rank matrix, we can obtain:
\begin{eqnarray}
X\beta=\sum_{i=1}^mX_{I_i}\beta_{I_i} = \sum_{i=1}^mU_iR_i\beta_{I_i} =\widetilde{X}\eta,\\
\|X_{I_i}\beta_{I_i}\|_2 = \|R_i\beta_{I_i}\|_2= \|\eta_{I_i}\|_2,
\end{eqnarray}
where   $\widetilde{X}_{I_i}=U_i$ and $\eta_{I_i}:=R_i\beta_{I_i}$ for $i = 1,\ldots,m$. Thus, the decoupled version of Problem (\ref{groupslope}) can be equivalently presented as:
\begin{eqnarray} \label{standardized0}
 \min\limits_{\eta} \frac12\|y-\widetilde{X}\eta\|_2^2+  J_{\lambda}( W\iI{\eta}) , 
\end{eqnarray}
where $\iI{\eta}:= \big(\|\eta_{I_1}\|_2, \|\eta_{I_2}\|_2,\ldots, \|\eta_{I_m}\|_2\big)^\top$.  

Considering the diagonal matrix $W$ with $W_{i,i} = w_i$ for $i=1,\ldots,m$, define $Z$ as a diagonal matrix with $Z_{i,i}:=1/w_j$ for $ i\in I_j$ where $i=1,\ldots,d$ and $j=1,\ldots,m$, we have
$
J_{\lambda}( W\iI{\eta}) = J_{\lambda}( \iI{Z^{-1}\eta}). 
$
Further, defining $b = Z^{-1}\eta$, we have $\eta =Z b$.   Thus, denoting $\hat{X}=\widetilde{X}Z$, we can formulate (\ref{standardized0}) as 
\begin{eqnarray} \label{standardized}
&&\min\limits_{\eta} \frac12\|y-\widetilde{X}\eta\|_2^2+  J_{\lambda}( W\iI{\eta})  \nonumber\\
& = &\min\limits_{b} \frac12\|y-\widetilde{X}\eta\|_2^2+  J_{\lambda}( \iI{Z^{-1}\eta}) \nonumber\\
& = &\min\limits_{b} \frac12\|y-\widetilde{X}Z b\|_2^2+  J_{\lambda}(\iI{b})\nonumber\\
& = &\min\limits_{b} \frac12\|y-\hat{X} b\|_2^2+  J_{\lambda}(\iI{b}).
\end{eqnarray}
That is to say, by the equivalent transformation above, the next step to achieve our aim is to propose a safe screening rule  for Problem (\ref{standardized}).

\subsection{Dual Formulation and Screening Test}
Problem (\ref{standardized}) can be formulated as follows:
\begin{eqnarray} \label{general}
{b} =  \argmin\limits_{b \in \mathbb{R}^{d} }  P_{\lambda}(b) := F(b) +J_{\lambda}(\iI{b}),
\end{eqnarray}
where loss $F(b) = \sum_{i=1}^n f_i(x_i^\top b)$, $f_i: \mathbb{R}\rightarrow \mathbb{R}_+$ is the squared loss. Generally,  (\ref{general}) is convex, non-smooth, and non-separable.

We initiate the derivation of the screening test by reformulating the primal objective (\ref{general}) into its dual. Leveraging insights from the dualization of $l_1$-regularized models as outlined in \cite{johnson2015blitz}, the resulting dual problem takes the form:
\begin{eqnarray} 
&&\min\limits_{b} F(b) + J_{\lambda}(\iI{b}) \label{general1} \nonumber \\
&=& \min\limits_{b}  \sum_{i=1}^n f_i(x_i^\top b) + J_{\lambda}(\iI{b}) \nonumber \\
& = & \min\limits_{b}   \sum_{i=1}^{n} f^{**}_{i}(x_{i}^\top b) +\sum_{i=1}^m\lambda_{i}\|b_{I_{[i]}}\|_2 \nonumber \\
& = & \min\limits_{b}   \sum_{i=1}^{n} \max\limits_{\theta_{i}} [\beta x_{i} \theta_{i}-f^{*}_{i}(\theta_{i})]  +\sum_{i=1}^m\lambda_{i}\|b_{I_{[i]}}\|_2 \nonumber \\
& = & \min\limits_{b}  \max\limits_{\theta} - \sum_{i=1}^{n} f^{*}_{i}(\theta_{i}) + \beta^\top X^\top \theta +\sum_{i=1}^m\lambda_{i}\|b_{I_{[i]}}\|_2\nonumber \\
& = &
\max\limits_{\theta} - \sum_{i=1}^{n} f^{*}_{i}(\theta_{i}) + \min\limits_{b}  \beta^\top X^\top \theta  + 
 \sum_{i=1}^m \lambda_{i}\|b_{I_{[i]}}\|_2 \nonumber \\
& = & \max\limits_{\theta \in  \Delta}  D(\theta) := \sum_{i=1}^{n} - f^{*}_{i}(\theta_{i}),  
\end{eqnarray}
where $\theta \in \mathbb{R}^{n} $  is the solution of the dual problem. Note $f^{*}_{i}$ is the convex conjugate of function $f_{i}$ as
\begin{eqnarray} 
f^{*}_{i}(\theta_{i}) = \max\limits_{z_{i} \in \mathbb{R}}  \theta_{i}z_{i} - f_{i}(z_{i}). 
\end{eqnarray}
Let us define $\widetilde{\theta}:=\big(\|X^\top_{I_1} \theta\|_2, \ldots, \|X^\top_{I_m} \theta\|_2\big)^\mathsf{T}$. Under this definition, the constraint $\theta \in  \Delta$ in (\ref{general1}) can be equivalently expressed as $\sum_{j \leq i} \widetilde{\theta}_{[j]} \leq \sum_{j \leq i} \lambda_j $ for all $i = 1, \ldots, m$. We next apply the optimality condition associated with the minimization part in the penultimate expression of (\ref{general1}): 
\begin{eqnarray} \label{optimality_origin}
\min\limits_{b}  \beta^\top X^\top \theta  + 
 \sum_{i=1}^m \lambda_{i}\|b_{I_{[i]}}\|_2.
\end{eqnarray}
This optimality condition naturally leads to the constraint structure previously introduced in (\ref{general1}), thereby finalizing the dual reformulation of (\ref{general}).

Leveraging the Fermat rule \cite{bauschke2011convex}, we obtain 
\begin{eqnarray}  \label{subdiff}
    -X^\top \theta^* \in \partial J_{\lambda}(\iI{b}), 
\end{eqnarray}
where $\theta^*$ denotes the dual optimum solution and $\partial J_{\lambda}(\iI{b})$ represents the subdifferential of the regularizer $J_{\lambda}(\iI{b})$. 

Let $\widetilde{\A}^*$ be the index corresponding to inactive groups at optimality. The conditions for the partition $\A^*$ and  $\widetilde{\A}^*$ of problems (\ref{optimality_origin}) can be separately expressed as:
\begin{eqnarray} \label{constraint2}
    -X_{I_{\A^*}}^\top\theta^{*}   \in \partial J_{\lambda_{\A^*}}(\iI{b}) ,   \\
      -X_{I_{\widetilde{\A}^*}}^\top\theta^{*}   \in \partial J_{\lambda_{\widetilde{\A}^*}}(\iI{b}).
\end{eqnarray}
For any group $i \in \A^*$, we have $b_{I_i}^* \neq 0$, it holds that 
\begin{eqnarray}  \label{optimal}
     \|X_{I_i}^\top \theta^* \|_2  \in [\min_{j\in  \A^*}\lambda_j, \max_{j\in \A^*}\lambda_j]. 
\end{eqnarray}
Assuming both primal and dual optimal solutions are available, one can derive a safe screening rule for each group based on the following condition:
\begin{eqnarray} \label{condition}
    \|X_{I_i}^\top \theta^* \|_2 < \lambda_{|\A^*|} \Longrightarrow b_{I_{i}}^* = 0,
\label{condition00}
\end{eqnarray}
which implies that such a group can be safely discarded without affecting the final solution. This enables subsequent training stages to proceed with a significantly reduced parameter space, leading to faster training while preserving accuracy.

Nevertheless, the main difficulty lies in the fact that the screening conditions (\ref{condition00}) necessitate prior knowledge of both the dual optimum and the order structure of the primal optimum, which can only be obtained after the full training process has been completed. As a result, these screening conditions cannot be utilized to enhance optimization during the training phase.

Therefore, our objective is to devise a screening rule capable of identifying as many inactive variables (i.e., those with coefficients that should be zero) as possible, using the screening test (\ref{condition00}) without knowing the dual optimum or the order structure of the primal optimum during the optimization process. To this end, we can formulate safe screening rules by defining a screening region that is as large as possible, characterized by smaller lower bounds and larger upper bounds derived from the screening conditions (\ref{condition00}).

\subsection{Upper Bound for the Left Term}
It is worth noting that the lower bound of the screening region corresponds to the upper bound of $\|X_{I_i}^\top \theta^* \|_{2}$. To this end, we focus on deriving a tight upper estimate for $\|X_{I_i}^\top \theta^* \|_{2}$ 
by monitoring the intermediate duality gap $G(b, \theta)$ throughout the training iterations.

Utilizing the triangle inequality, we have:
\begin{eqnarray} \label{triangle}
 \|X_{I_i}^\top \theta^* \|_{2}\leq  \|X_{I_i}^\top \theta \|_{2} +  \|X_{I_i} \|_{2}\|\theta - \theta^*\|_2. 
\end{eqnarray}

Since each $f_i^*(\theta_{i})$ in the dual is known to be strongly convex (see Proposition 3.2 in\cite{johnson2015blitz}), the overall dual objective $D(\theta):= \sum_{i=1}^{n} - f^{*}_{i}(\theta_{i})$ inherits strong concavity \emph{w.r.t.} $\theta$. As a direct implication, we obtain the following upper bound:
\begin{eqnarray} \label{concave}
D(\theta) \leq D(\theta^*)- \mathrm{tr}(\nabla D(\theta^*)^\top(\theta^* - \theta))   - \frac{1}{2}\| \theta - \theta^* \|^2_{2}.
\end{eqnarray}
This inequality enables us to derive a bound on the distance between any feasible dual iterate $\theta$ and the optimal dual solution $\theta^*$ based on the first-order condition summarized in Corollary \ref{corollary1}

\begin{corollary} \label{corollary1}
For any dual feasible point $\theta$, the following estimate holds:
\begin{eqnarray}
    \| \theta - \theta^* \|_2 \leq \sqrt{2G(b, \theta)},
\end{eqnarray}
where $G(b, \theta) = P(b) - D(\theta)$ denotes the intermediate duality gap at training.
\end{corollary}
\begin{proof}
We begin by applying the first-order optimality condition to the strongly concave dual objective $D(\theta)$, yielding:
\begin{eqnarray} 
    \mathrm{tr}(\nabla D(\theta^*)^\top(\theta^* - \theta)) \geq 0.
\end{eqnarray}
Combining this with inequality (\ref{concave}), we obtain:
\begin{eqnarray} 
\|\theta-\theta^*\|_2\leq\sqrt{D(\theta^*) - D(\theta))}.
\end{eqnarray}
Under the assumption of strong duality, which ensures $P(b ) \geq D(\theta^*)$, we replace the intractable term with a computable surrogate:
\begin{eqnarray} 
    \| \theta - \theta^* \|_2  \leq \sqrt{2(P(b ) - D(\theta))}.
\end{eqnarray}
This concludes the proof.
\end{proof}

Based on the upper bound derived in Corollary \ref{corollary1}, we substitute the quantity $\|\theta - \theta^* \|_2$ in the right-hand side of Inequality (\ref{triangle}).  This leads to an improved safe screening condition given by:
\begin{eqnarray} \label{test}
\|X_{I_i}^\top \theta \|_{2}   +   \|X_{I_i}\|_{2}\sqrt{2G(b, \theta)} < \lambda_{|\A^*|}  \Rightarrow b_{I_{i}}^* = 0. 
\end{eqnarray}
The duality gap $G(b, \theta)$ can be efficiently evaluated using the primal-dual variables $b$ and $\theta$, both of which are directly available during each iteration of standard proximal gradient methods.

As training proceeds and the gap between the primal and dual objectives narrows, the upper estimate of $\|X_{I_i}^\top \theta^* \|_{2}$ is reduced accordingly. This results in a progressively tighter screening threshold over time.

\subsection{Lower bound for the Right Term}
In contrast, the upper limit of the screening region aligns with the minimal value of $\lambda_{|\A^*|} $. Therefore, our objective in this section is to derive a sharp lower estimate for this critical quantity.

To effectively compute this bound amidst numerous unspecified hyperparameters, we design an iterative scheme that tackles the challenges introduced by the non-separability of the penalty term. This is achieved by exploiting the unknown order structure embedded in the primal solution of (\ref{condition00}). In its general form, our screening criterion can be formulated as:
\begin{eqnarray} \label{screening}
\|X_{I_i}^\top \theta \|_{2}   +   \|X_{I_i} \|_{2}\sqrt{2G(b, \theta)} < \lambda_{|\A|} \Rightarrow b^*_{I_{i}} = 0.
\end{eqnarray}

At the initial stage, we assume all $m$ groups are active, and hence the screening is applied with respect to $\lambda_m$:
\begin{eqnarray} 
\|X_{I_i}^\top \theta \|_{2}   +   \|X_{I_i}\|_{2}\sqrt{2G(b, \theta)} < \lambda_m \Rightarrow b^*_{I_{i}} = 0. 
\end{eqnarray}
As training proceeds and only $m_k$ groups remain active, the set $\mathcal{A}$ is updated to size $m_k$. The remaining $m - m_k$ groups can then be assigned any permutation of $\lambda_{m_k+1}, \ldots, \lambda_m$—the smallest parameters—without affecting the final result. This reveals that the ranks of the $m - m_k$ zero-valued coefficients are deterministically among the lowest values in the $\lambda$ sequence. Accordingly, the screening test is adapted and evaluated at $\lambda_{m_k}$:
\begin{eqnarray} 
\|X_{I_i}^\top \theta \|_{2}   +   \|X_{I_i} \|_{2}\sqrt{2G(b, \theta)} < \lambda_{m_k} \Rightarrow b^*_{I_{i}} = 0, 
\end{eqnarray} 
yielding an updated active group set $\mathcal{A}$ of size $m_k'$, where $m_k' \leq m_k$. The $m_k - m_k'$ newly deactivated groups are again assigned the next smallest unused $\lambda$ values.

This iterative refinement continues until convergence—i.e., when the active set stabilizes. Crucially, each iteration involves only a single hyperparameter, ensuring computational efficiency even when Group SLOPE involves a large number of tuning parameters.

Throughout this process, as the active set $\A$ shrinks and due to the monotonicity of $\lambda$, the corresponding $\lambda_{|\A|}$ increases. This raises the lower bound of $\lambda_{|\A^*|}$ and in turn, enlarges the screening threshold.

In summary, by jointly updating the upper and lower bounds across iterations, the screening region continuously expands, enabling more inactive groups to be excluded and improving overall algorithmic efficiency.

For the computation of the screening rule,  the dual $f_i^*$ for Group SLOPE can also be calculated as $f_i^*(\theta_i) = \frac{1}{2}\theta_i^2+\theta_i y_i $.
Using this screening rule, inactive feature groups can be eliminated during the training process.

\section{Proposed Algorithms}
In this section, we begin by applying the safe screening rules to proximal gradient algorithms, specifically focusing on the APGD algorithm for batch settings and the SPGD algorithm for stochastic settings. We then delve into the theoretical analysis of our unified safe screening rules, highlighting their properties in terms of safeness, convergence, and screening capability.

\subsection{Algorithms}
Coordinate descent and block coordinate descent methods are efficient for solving Lasso and Group Lasso problems. However, due to the challenges posed by the non-separable penalty, these algorithms are highly efficient but not practical for solving Group SLOPE. To address this issue, accelerated proximal gradient descent (APGD) methods have been proposed, as seen in works like \cite{brzyski2019group,gossmann2015identification}.

Since Group SLOPE is generally used in high-dimensional settings, all the proximal algorithms mentioned above face significant computational and memory challenges when dealing with large feature sizes. Therefore, accelerating the training of Group SLOPE through the use of screening techniques for proximal algorithms becomes both important and promising.

For the APGD algorithm in batch settings, our approach involves repeatedly performing the screening test and updating the active set $\A$. If $\A$ is updated during this iteration, we set the step size $t_k = t_1$. From this point onward, the procedure mirrors that of the original APGD algorithm, using the current active set $\A$. This process is detailed in Algorithm \ref{algAPGDScreen}.

\begin{algorithm}[ht] 
\renewcommand{\algorithmicrequire}{\textbf{Input:}}
\renewcommand{\algorithmicensure}{\textbf{Output:}}
\caption{APGD Algorithm with Our Safe Screening Rules}
\begin{algorithmic}[1]
\REQUIRE $b^{0},\hat{b}^{1}=b^{0},t_{1}=1$
\FOR{$k=1,2,\ldots$}
\REPEAT 
\STATE Apply the safe screening from (\ref{screening})
\STATE Update the active group set $\A$
\UNTIL $\A$ remains unchanged
\IF{$\A$ was updated}
\STATE $t_{k} = t_{1}$
\ENDIF
\STATE $b^{k} = prox_{t_{k},{\lambda}}(\hat{b}^{k}-t_{k}\nabla F(\hat{b}^k))$
\STATE $t_{k+1}=\frac{1}{2}(1+\sqrt{1+4 t_{k}^{2}})$
\STATE $\hat{b}^{k+1} = b^{k} + \frac{t_{k}-1}{t_{k+1}}(b^{k}-b^{k-1})$
\ENDFOR
\ENSURE Coefficient $b$
\end{algorithmic}
\label{algAPGDScreen}
\end{algorithm}

\begin{algorithm}[ht]
\renewcommand{\algorithmicrequire}{\textbf{Input:}}
\renewcommand{\algorithmicensure}{\textbf{Output:}}
\caption{SPGD Algorithm with Our Safe Screening Rules}
\begin{algorithmic}[1]
\REQUIRE $b^{0},l$.
\FOR{$k=1,2,\ldots$}
\REPEAT 
\STATE Apply the safe screening from (\ref{screening})
\STATE Update the active group set $\A$
\UNTIL $\A$ remains unchanged
\STATE $b = b^{k-1}$
\STATE $\tilde{v} = \nabla F(b)$
\STATE $\tilde{b}^{0} = b$
\FOR{$t=1,2,\ldots,T$}
\STATE Pick mini-batch $I_{t} \subseteq X$ of size $l$
\STATE $v_{t}=(\nabla F_{I_{t}}(\tilde{b}^{t-1}) - \nabla F_{I_{t}}(b))/l + \tilde{v}  $
\STATE $\tilde{b}^{t} = prox_{\gamma,{\lambda}}(\tilde{b}^{t-1}-\gamma v_{t})$
\ENDFOR
\STATE $b^{k} = \tilde{b}^{T}$
\ENDFOR
\ENSURE Coefficient $b$
\end{algorithmic}
\label{algSPGDScreen}
\end{algorithm}

Moreover, as each update in Algorithm \ref{algAPGDScreen} relies on all samples, the per-iteration cost of the APGD algorithm can be substantial in large-scale learning because it necessitates full gradient computations. To mitigate this, the stochastic proximal gradient descent (SPGD) algorithm, as introduced in \cite{xiao2014proximal} and building on \cite{johnson2013accelerating}, serves as an efficient alternative in the stochastic setting, requiring only mini-batch gradient calculations.

In applying our screening rule to the SPGD algorithm for stochastic settings, we similarly repeat the screening test and update $\A$ in the outer loop before proceeding with the standard SPGD algorithm steps using the newly obtained active set. This procedure is outlined in Algorithm \ref{algSPGDScreen}.

Interestingly, the duality gap, which represents the main time-consuming aspect of our screening rule, has already been computed in the original APGD and SPGD algorithms. Furthermore, as inactive variables are continually screened during optimization, and given that the active set size for iteration $k$ is $d_{k}$, the computational complexity of the screening rule for this iteration is only $O(d_{k})$, which is even less than the complexity of the original stopping criterion evaluation $O(d)$. Consequently, the complexity $O(d_{k}(n+\log d_{k}))$ and $O(d_{k}(n+ T l + T \log d_{k}))$ for each iteration of the APGD and SPGD algorithms, respectively, can be reduced to $O(d_{k})$ for analysis purposes.

We further examine the overall complexity of our algorithms, focusing on both the per-iteration cost and the number of iterations. For per-iteration cost, if the algorithm has $d_{k}$ active variables at iteration $k$, our Algorithm \ref{algAPGDScreen} requires only $O(d_{k}(n+\log d_{k}))$, which is less than the original APGD algorithm's complexity of $O(d(n+\log d))$. Regarding the number of iterations, since the optimal solution for the inactive features screened at iteration $k$ must be zeros, removing these inactive features beforehand either keeps the objective function the same or decreases it. Thus, our Algorithm \ref{algAPGDScreen} will converge to the same stopping criterion with at most the same (and usually fewer) iterations compared to the original APGD algorithm. With fewer or the same iterations and lower per-iteration costs, our proposed algorithm is more efficient. Similarly, our Algorithm \ref{algSPGDScreen} requires $O(d_{k}(n+ T l + T \log d_{k}))$ for the main loop $k$, whereas the original SPGD algorithm's complexity is $O(d(n+T l + T \log d))$. Since our algorithm also requires fewer or the same iterations and lower per-iteration costs, the overall complexity is reduced compared to the original SPGD algorithm.

More precisely, the computational advantage of our methods hinges on the sparsity of the final model. The training process gains more from the screening rule when dealing with sparser models. In cases where $n < d$, the final model becomes very sparse, leading to $d_{k} \ll d $ during training. Consequently, our method is particularly well-suited for high-dimensional settings. It is evident that the proposed algorithms consistently outperform the original ones in terms of speed. Additionally, it's important to note that our method also applies to datasets where $n>d$. Assuming the presence of sparsity, our screening rule will effectively identify inactive features, enabling our algorithms to determine the final active set in a finite number of iterations. As a result, we still achieve $d_{k} < d $ during training, and the per-iteration cost of our algorithm remains lower than the original one. Therefore, with fewer or an equivalent number of iterations and reduced per-iteration costs, the proposed algorithms remain faster than the original ones for $n>d$.

The following part provides the theoretical analysis of our screening rules, highlighting their safeness, convergence, and screening ability.

\begin{property}
(Safeness) The proposed screening rule retains all relevant groups throughout the entire optimization trajectory of Group SLOPE, irrespective of the specific iterative method employed.
\label{propertysafe}
\end{property}

\begin{property} \label{propertyconverge}
(Convergence) Our screening rule can be seamlessly embedded into a wide range of iterative algorithms, such as APGD, SPGD, and their derivatives, without disrupting convergence guarantees.
\end{property}

\begin{theorem} \label{propertyscreen}
Let $i$ denote any group that belongs to the final active set $\A^{*}$. Then $ \|X_{I_i}^\top \theta^* \|_{2} \in [\min_{j\in  \A^*}\lambda_j, \max_{j\in \A^*}\lambda_j] $. As algorithm $\Psi$ converges, there exists a finite iteration number $K_{0} \in \mathbb{N}$ s.t. $\forall k \geq K_{0}$, any group $i \notin \A^{*}$ will be successfully discarded by the screening rule.
\end{theorem}

\begin{proof}
From the strong concavity of the dual problem, it follows that the optimal dual variable $\theta^*$ is unique. Moreover, $\theta$ converges to $\theta^*$ as $b$ converges to $b^*$. Thus, for any given $\epsilon > 0$, one can find an index $K_{0}$ s.t. $\forall k \geq K_{0}$:
$
    \| \theta^{k} - \theta^* \|_{2} \leq  \epsilon, 
 \sqrt{2G(b^k, \theta^k)}\leq  \epsilon. 
$ Then, for any group $i \notin \A^{*}$, we can bound the screening condition as:
\begin{eqnarray} 
&&  \|X_{I_i}^\top \theta^k \|_{2} +  \|X_{I_i} \|_{2}\sqrt{2G(b^k, \theta^k)} \nonumber\\
& \leq & \|X_{I_i} \|_{2}\|\theta^k-\theta^*\|_{2} + \|X_{I_i}^\top \theta^* \|_{2}   +   \|X_{I_i} \|_{2}\sqrt{2G(b^k, \theta^k)} \nonumber\\
& \leq &  2 \|X_{I_i} \|_{2} \epsilon + \|X_{I_i}^\top \theta^* \|_{2}.
\end{eqnarray}
Since  $i \notin \A^{*}$ implies $\lambda_{|\A^{*}|}-\|X_{I_i}^\top \theta^* \|_{2} > 0 $, choosing
$
\epsilon < \frac{\lambda_{|\A^{*}|}-\|X_{I_i}^\top \theta^* \|_{2} }{2\|X_{I_i} \|_{2}},
$
ensures
$
\|X_{I_i}^\top \theta^k \|_{2} +  \|X_{I_i}\|_{2}\sqrt{{2G(b^k, \theta^k)}}  < \lambda_{|\A^{*}|}, 
$
which triggers our screening condition. 
\end{proof}

Theorem \ref{propertyscreen} highlights the strong screening performance of the proposed rules. As the iterative solver progresses and the duality gap narrows, the rule becomes increasingly effective: the upper bound on the left-hand side tightens, while the right-hand side’s lower bound increases. This progressively improves the chances of filtering out inactive groups. Ultimately, every group $i \notin \A^{*}$ will be accurately screened and discarded in a finite number of iterations.

\section{Experiments}
In this section, we outline our experimental setup and subsequently present the results and discussions.

\subsection{Experimental Setup}

\subsubsection{Design of Experiments}
We empirically evaluated our method on real-world benchmark datasets under the Group SLOPE framework, highlighting its computational advantages and its capability to reliably discard irrelevant groups.

To assess the efficiency of our algorithms in reducing computation time, we compared the runtime of our proposed algorithms against other competitive algorithms for solving Group SLOPE under various conditions. Given that the APGD algorithm is well-suited for scenarios where $n \ll p$ in the batch setting, and the SPGD algorithm is tailored for large-scale learning where $n$ is large in stochastic settings, we evaluated the runtime across both batch and stochastic setups using different datasets. The algorithms compared in batch and stochastic settings are summarized as follows: 
\begin{itemize} 
\item Batch setting 
\begin{itemize} 
\item APGD: Accelerated proximal gradient descent algorithms as presented in \cite{gossmann2015identification,brzyski2019group}. 
\item APGD + Screening: Accelerated proximal gradient descent algorithms enhanced with our safe screening rules. \end{itemize} 
\item Stochastic setting 
\begin{itemize} \item SPGD: Stochastic proximal gradient descent algorithm adopted from \cite{xiao2014proximal}. \item SPGD + Screening: Stochastic proximal gradient descent algorithm integrated with our safe screening rules. \end{itemize} 
\end{itemize}

\begin{table}[t]
\caption{The descriptions of benchmark datasets used in our experiments.}
\begin{center}
\begin{tabular}{lcc}
\toprule
\textbf{Dataset} & \textbf{Sample size} & \textbf{Attribute} \\
\midrule

     Duke Breast Cancer (DBC) & 44  & 7129   \\
    
     Colon Cancer (CC)  & 62 & 2000    \\

     IndoorLoc (IL) & 21048 & 520  \\

     SenseIT Vehicle (SV) & 78823 & 100 \\

\bottomrule
\end{tabular}

%\end{small}
\end{center}
%\vskip -0.1in
	\label{table:datasets}
\end{table}

To further confirm the effectiveness of our algorithms in filtering inactive variables, we evaluated the screening rate at each iteration of the algorithms with our screening rules applied to Group SLOPE, tested in both batch and stochastic setups across various datasets during the training process.

\subsubsection{Datasets}
Table \ref{table:datasets} provides an overview of the benchmark datasets utilized in our experiments. Duke Breast Cancer, Colon Cancer, and SenseIT Vehicle datasets are from the LIBSVM repository \cite{chang2011libsvm}, which can be accessed at \url{https://www.csie.ntu.edu.tw/~cjlin/libsvmtools/datasets/}.  IndoorLoc dataset is obtained from the UCI benchmark repository \cite{Dua:2019}, available at \url{https://archive.ics.uci.edu/ml/datasets.php}. IndoorLoc dataset includes $2$ tasks: IndoorLoc Latitude and IndoorLoc Longitude.

\begin{figure}[t]
  \centering
\subfigure[DBC]{\includegraphics[width=0.45\textwidth]{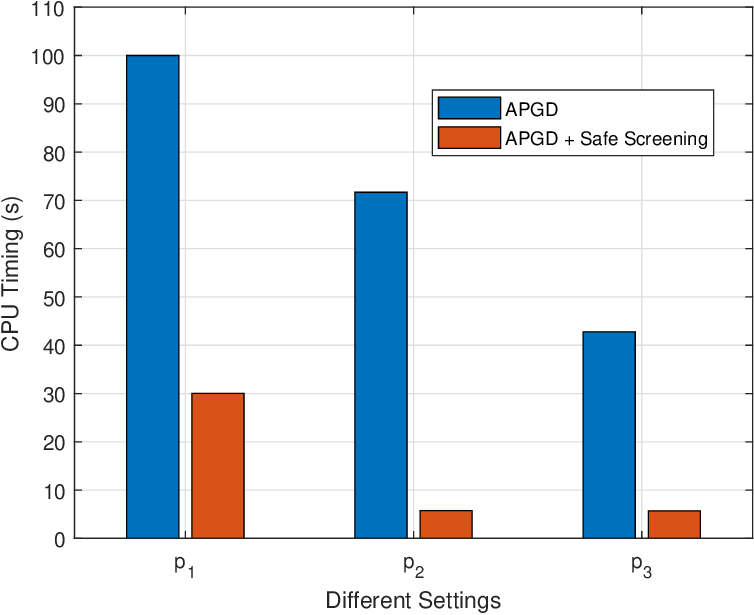}}
        \hspace{8mm}
\subfigure[CC]{\includegraphics[width=0.45\textwidth]{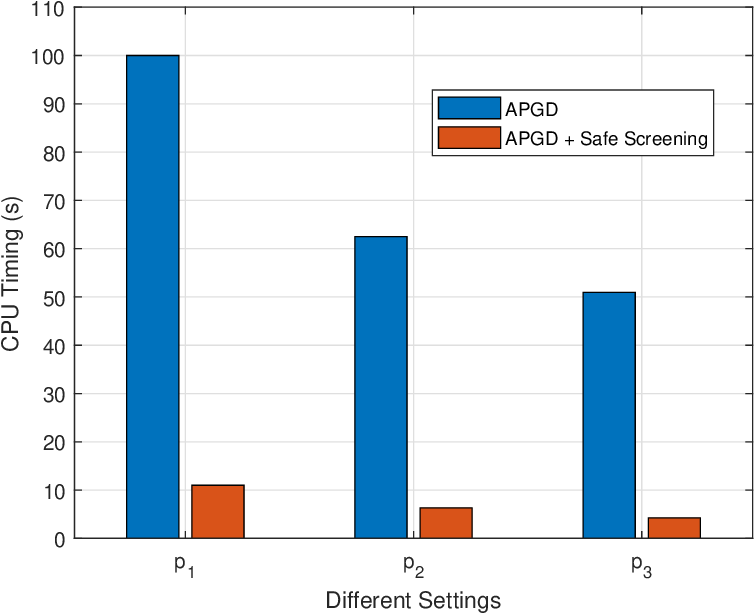}}
        \hspace{8mm}
\subfigure[IL]{\includegraphics[width=0.45\textwidth]{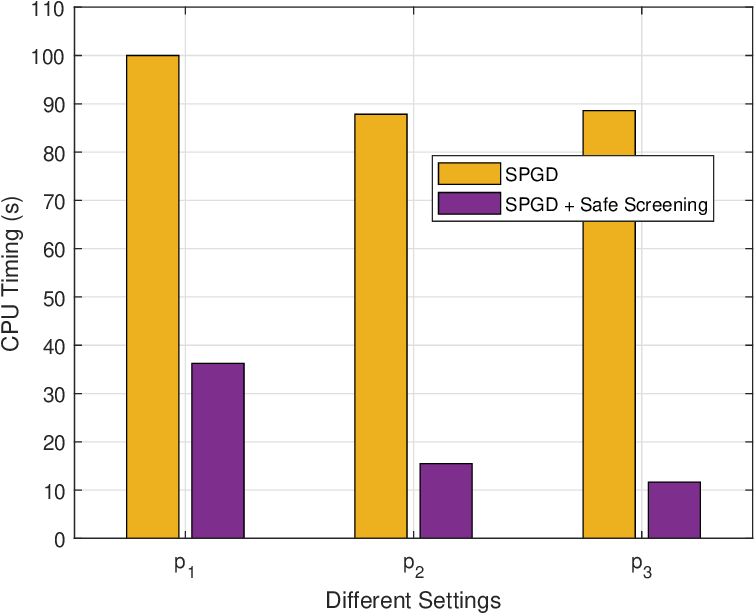}}
        \hspace{8mm}
\subfigure[SV]{\includegraphics[width=0.45\textwidth]{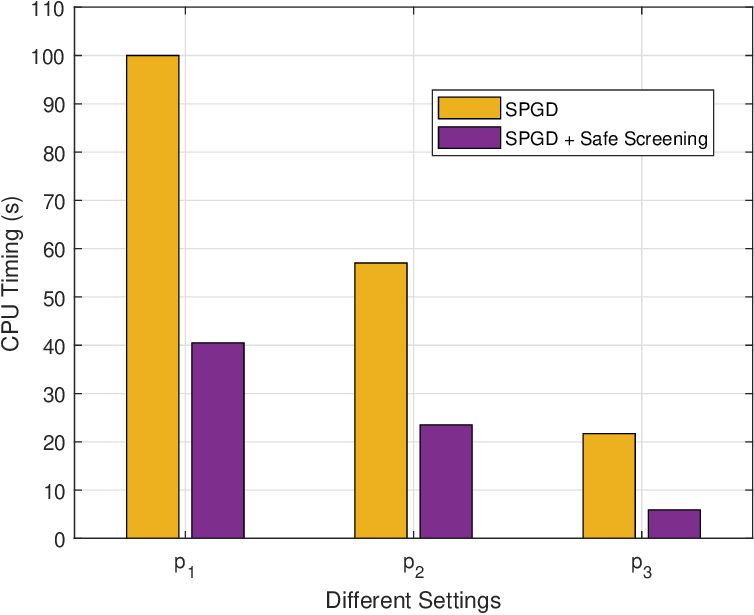}}
  \caption{Running time of the algorithms without and with safe screening for Group SLOPE.}
\label{fig21} 
\end{figure}

\begin{figure*}[t]
  \centering
\subfigure[DBC]{\includegraphics[width=0.45\textwidth]{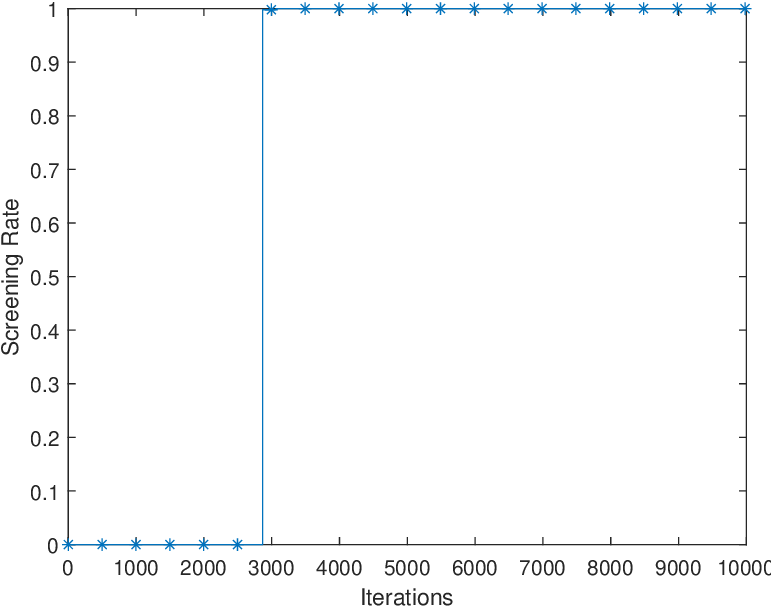}}
        \hspace{8mm}
\subfigure[CC]{\includegraphics[width=0.45\textwidth]{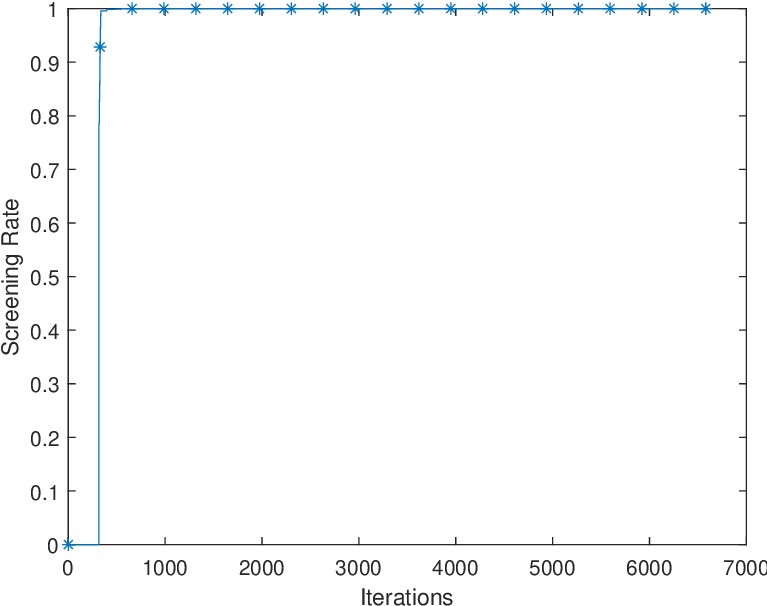}}
        \hspace{8mm}
\subfigure[IL]{\includegraphics[width=0.45\textwidth]{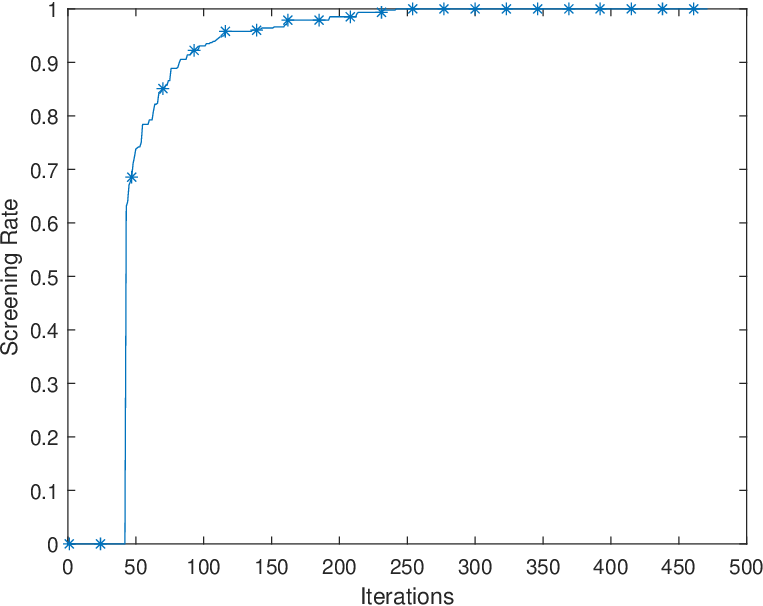}}
        \hspace{8mm}
\subfigure[SV]{\includegraphics[width=0.45\textwidth]{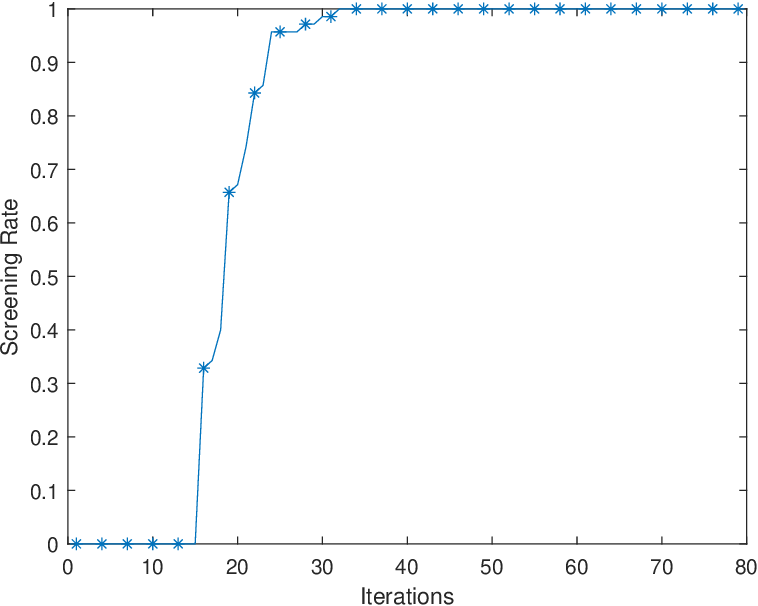}}
  \caption{Screening rate of our screening rule in both batch and stochastic settings for Group SLOPE.}
\label{fig22} 
\end{figure*}

\subsubsection{Implementation Details}
We implemented all the algorithms using MATLAB and compared the average CPU time across different algorithms on a 2.70 GHz machine over 5 trials. We adhered to the basic setup described in \cite{brzyski2019group} and set the tolerance for the duality gap and the dual infeasibility to $10^{-6}$. To ensure fairness in comparisons, the experimental configurations for Algorithm \ref{algAPGDScreen} and \ref{algSPGDScreen} followed the original APGD and SPGD algorithms, maintaining consistent hyperparameters across all setups. In the stochastic setting, the mini-batch size and the number of inner loop iterations were set to 30 or 40, depending on the dataset. The step size $\gamma$ was selected from a range of $10^{-8}$ to $10^{-5}$. Initially, the APGD algorithms exhibited a large duality gap, offering minimal benefit from our screening rule, so we first ran the algorithms without screening and later applied our screening rule with a warm start. For ease of comparison, the CPU time of each algorithm is presented as a percentage relative to the runtime of the first configuration for each dataset.

The OSCAR hyperparameter setting, which is commonly used (see \cite{oswal2016representational,zhong2012efficient,zhang2018learning,bao2019efficient}), was applied in all our experiments:
\begin{eqnarray} 
    \lambda_{i} = \alpha_{1} + \alpha_{2}(m-i),
\end{eqnarray}
where $\alpha_{1} = p_{i} \|X^\top y\|_{\infty}$ and  $\alpha_{2} = \alpha_{1}/d$ for Group SLOPE. For a fair comparison, the factor $p_{i}$ is used to control sparsity. In our experiments, we set $p_{i} = i * e^{-\tau}$, $i = 1,2, 3$. 

For Group SLOPE, batch algorithms were applied to the Duke Breast Cancer and Colon Cancer datasets, while stochastic algorithms were run on the SenseIT and IndoorLoc Longitude datasets, with $\tau$ set to 2 for IndoorLoc Longitude and 3 for the other datasets. We duplicate each feature $i$ to form the feature group  $I_i$ with size $|I_i| \sim U(1, s)$ where $U$ is the discrete uniform distribution and choose $s =10$ for Duke Breast Cancer and Colon Cancer datasets and $s =40$ for IndoorLoc Latitude and SenseIT Vehicle datasets. 

To assess the screening rate of our algorithms, it was calculated as the proportion of inactive partitions of groups screened by our method to the total number of inactive groups at the optimal solution. We used the $p_{1}$ setting for this evaluation.

\subsection{Experimental Results and Discussions}

Figures \ref{fig21}(a)–(d) present the average runtime comparisons of the proposed algorithms with and without the safe screening technique applied to Group SLOPE, under both batch and stochastic optimization frameworks across multiple experimental settings. In cases where the sample size is much smaller than the number of features ($n \ll d$), the APGD method achieves acceleration ratios between 3× and 14× when integrated with our screening approach. For large-scale problems, incorporating the screening rule into the SPGD variant leads to speed improvements ranging from 2.5× to 8× compared to its unscreened counterpart. These observations consistently validate the substantial efficiency gains attained by augmenting Group SLOPE solvers—both batch and stochastic—with our screening mechanism. The primary contributor to this acceleration is the effective early-stage exclusion of inactive feature groups, which lowers the computational complexity throughout training. Notably, the improvements become even more pronounced as the data dimensionality increases and sparsity intensifies.

Figures \ref{fig22}(a)-(d) depict the screening rates of our algorithms in both batch and stochastic settings for Group SLOPE, highlighting the effectiveness and characteristics of our screening rule. The data support the conclusion that our algorithm can successfully eliminate most inactive feature groups at an early stage, converge on the final active set, and ultimately screen nearly all inactive features within a finite number of iterations. This effectiveness is due to the tight upper and lower bounds of the screening test's left and right terms, respectively, which allow for more efficient screening of inactive variables as the optimization algorithm progresses. Specifically, as the algorithm converges, the duality gap narrows, leading to a continuously decreasing upper bound for the left term of the screening test, while the iterative strategy increasingly solidifies the order structure of variables, thus continuously increasing the lower bound for the right term.

\section{Conclusion}\label{Conclusion}
In this paper, we introduced a safe variable screening rule for Group SLOPE, addressing the challenges posed by the non-separable group effects. This approach significantly speeds up the training process by eliminating unnecessary computations for inactive variables. Our screening rule is uniquely dynamic, featuring a decreasing left term through tracking the intermediate duality gap, and an increasing right term by iteratively assessing the order of the primal solution, considering the unknown order structure. Importantly, the proposed rules are seamlessly integrable into existing iterative optimization methods, applicable in both batch and stochastic settings, such as the APGD and SPGD algorithms. We have theoretically proven that our screening rule remains safe when applied to these algorithms, ensuring no loss in accuracy. Extensive empirical results on real-world benchmark datasets demonstrate that our algorithms provide substantial computational benefits while maintaining accuracy in both batch and stochastic learning contexts by effectively screening out inactive variables.

\section*{Acknowledgement}
This work is supported in part by the National Science Foundation (NSF) grant IIS-2451436 and
Commonwealth Cyber Initiative grant HC-4Q24-059.

\bibliographystyle{splncs04}
\bibliography{mybib}

\begin{thebibliography}{10}
\providecommand{\url}[1]{\texttt{#1}}
\providecommand{\urlprefix}{URL }
\providecommand{\doi}[1]{https://doi.org/#1}

\bibitem{bao2019efficient}
Bao, R., Gu, B., Huang, H.: Efficient approximate solution path algorithm for order weight l\_1-norm with accuracy guarantee. In: 2019 IEEE International Conference on Data Mining (ICDM). pp. 958--963. IEEE (2019)

\bibitem{bao2020fast}
Bao, R., Gu, B., Huang, H.: Fast oscar and owl regression via safe screening rules. In: International conference on machine learning. pp. 653--663. PMLR (2020)

\bibitem{bao2022accelerated}
Bao, R., Gu, B., Huang, H.: An accelerated doubly stochastic gradient method with faster explicit model identification. In: Proceedings of the 31st ACM International Conference on Information \& Knowledge Management. pp. 57--66 (2022)

\bibitem{bao2025safe}
Bao, R., Lu, Q., Zhang, Y.: Safe screening rules for group owl models. arXiv preprint arXiv:2504.03152  (2025)

\bibitem{bao2022doubly}
Bao, R., Wu, X., Xian, W., Huang, H.: Doubly sparse asynchronous learning. In: The 31st International Joint Conference on Artificial Intelligence (IJCAI 2022) (2022)

\bibitem{bauschke2011convex}
Bauschke, H.H., Combettes, P.L., et~al.: Convex analysis and monotone operator theory in Hilbert spaces, vol.~408. Springer (2011)

\bibitem{bergersen2011weighted}
Bergersen, L.C., Glad, I.K., Lyng, H.: Weighted lasso with data integration. Statistical applications in genetics and molecular biology  \textbf{10}(1) (2011)

\bibitem{bogdan2015slope}
Bogdan, M., Van Den~Berg, E., Sabatti, C., Su, W., Cand{\`e}s, E.J.: Slope—adaptive variable selection via convex optimization. The annals of applied statistics  \textbf{9}(3),  667–698 (2015)

\bibitem{bonnefoy2015dynamic}
Bonnefoy, A., Emiya, V., Ralaivola, L., Gribonval, R.: Dynamic screening: Accelerating first-order algorithms for the lasso and group-lasso. IEEE Transactions on Signal Processing  \textbf{63}(19),  5121--5132 (2015)

\bibitem{brzyski2019group}
Brzyski, D., Gossmann, A., Su, W., Bogdan, M.: Group slope--adaptive selection of groups of predictors. Journal of the American Statistical Association  \textbf{114}(525),  419--433 (2019)

\bibitem{chang2011libsvm}
Chang, C.C., Lin, C.J.: Libsvm: A library for support vector machines. ACM transactions on intelligent systems and technology (TIST)  \textbf{2}(3),  1--27 (2011)

\bibitem{Dua:2019}
Dua, D., Graff, C.: {UCI} machine learning repository (2017)

\bibitem{fercoq2015mind}
Fercoq, O., Gramfort, A., Salmon, J.: Mind the duality gap: safer rules for the lasso. In: International Conference on Machine Learning. pp. 333--342 (2015)

\bibitem{feser2024strong}
Feser, F., Evangelou, M.: Strong screening rules for group-based slope models. arXiv preprint arXiv:2405.15357  (2024)

\bibitem{frankle2018lottery}
Frankle, J., Carbin, M.: The lottery ticket hypothesis: Finding sparse, trainable neural networks. In: International Conference on Learning Representations (2018)

\bibitem{gossmann2017sparse}
Gossmann, A., Cao, S., Brzyski, D., Zhao, L.J., Deng, H.W., Wang, Y.P.: A sparse regression method for group-wise feature selection with false discovery rate control. IEEE/ACM transactions on computational biology and bioinformatics  \textbf{15}(4),  1066--1078 (2017)

\bibitem{gossmann2015identification}
Gossmann, A., Cao, S., Wang, Y.P.: Identification of significant genetic variants via slope, and its extension to group slope. In: Proceedings of the 6th ACM Conference on Bioinformatics, Computational Biology and Health Informatics. pp. 232--240 (2015)

\bibitem{han2015learning}
Han, S., Pool, J., Tran, J., Dally, W.: Learning both weights and connections for efficient neural network. Advances in neural information processing systems  \textbf{28} (2015)

\bibitem{jacob2009group}
Jacob, L., Obozinski, G., Vert, J.P.: Group lasso with overlap and graph lasso. In: Proceedings of the 26th annual international conference on machine learning. pp. 433--440 (2009)

\bibitem{jenatton2011structured}
Jenatton, R., Audibert, J.Y., Bach, F.: Structured variable selection with sparsity-inducing norms. The Journal of Machine Learning Research  \textbf{12},  2777--2824 (2011)

\bibitem{johnson2013accelerating}
Johnson, R., Zhang, T.: Accelerating stochastic gradient descent using predictive variance reduction. In: Advances in neural information processing systems. pp. 315--323 (2013)

\bibitem{johnson2015blitz}
Johnson, T., Guestrin, C.: Blitz: A principled meta-algorithm for scaling sparse optimization. In: International Conference on Machine Learning. pp. 1171--1179 (2015)

\bibitem{kim2010tree}
Kim, S., Xing, E.P.: Tree-guided group lasso for multi-task regression with structured sparsity. In: Proceedings of the 27th International Conference on International Conference on Machine Learning. pp. 543--550 (2010)

\bibitem{larsson2020strong}
Larsson, J., Bogdan, M., Wallin, J.: The strong screening rule for slope. Advances in neural information processing systems  \textbf{33},  14592--14603 (2020)

\bibitem{Laurent2012safe}
Laurent El~Ghaoui, Vivian~Viallon, T.R.: Safe feature elimination in sparse supervised learning. Pacific Journal of Optimization  \textbf{8},  667–698 (2012)

\bibitem{lu2024all}
Lu, L., Wang, Z., Bao, R., Wang, M., Li, F., Wu, Y., Jiang, W., Xu, J., Wang, Y., Gao, S.: All-in-one tuning and structural pruning for domain-specific llms. arXiv preprint arXiv:2412.14426  (2024)

\bibitem{ndiaye2016gap}
Ndiaye, E., Fercoq, O., Gramfort, A., Salmon, J.: Gap safe screening rules for sparse-group lasso. In: Advances in Neural Information Processing Systems. pp. 388--396 (2016)

\bibitem{oswal2016representational}
Oswal, U., Cox, C., Lambon-Ralph, M., Rogers, T., Nowak, R.: Representational similarity learning with application to brain networks. In: International Conference on Machine Learning. pp. 1041--1049 (2016)

\bibitem{rakotomamonjy2019screening}
Rakotomamonjy, A., Gasso, G., Salmon, J.: Screening rules for lasso with non-convex sparse regularizers. In: International Conference on Machine Learning. pp. 5341--5350 (2019)

\bibitem{shibagaki2016simultaneous}
Shibagaki, A., Karasuyama, M., Hatano, K., Takeuchi, I.: Simultaneous safe screening of features and samples in doubly sparse modeling. In: International Conference on Machine Learning. pp. 1577--1586 (2016)

\bibitem{simon2013sparse}
Simon, N., Friedman, J., Hastie, T., Tibshirani, R.: A sparse-group lasso. Journal of computational and graphical statistics  \textbf{22}(2),  231--245 (2013)

\bibitem{tibshirani2012strong}
Tibshirani, R., Bien, J., Friedman, J., Hastie, T., Simon, N., Taylor, J., Tibshirani, R.J.: Strong rules for discarding predictors in lasso-type problems. Journal of the Royal Statistical Society: Series B (Statistical Methodology)  \textbf{74}(2),  245--266 (2012)

\bibitem{wang2014two}
Wang, J., Ye, J.: Two-layer feature reduction for sparse-group lasso via decomposition of convex sets. In: Advances in Neural Information Processing Systems. pp. 2132--2140 (2014)

\bibitem{wang2015multi}
Wang, J., Ye, J.: Multi-layer feature reduction for tree structured group lasso via hierarchical projection. In: Advances in Neural Information Processing Systems. pp. 1279--1287 (2015)

\bibitem{wang2014safe}
Wang, J., Zhou, J., Liu, J., Wonka, P., Ye, J.: A safe screening rule for sparse logistic regression. In: Advances in neural information processing systems. pp. 1053--1061 (2014)

\bibitem{wang2013lasso}
Wang, J., Zhou, J., Wonka, P., Ye, J.: Lasso screening rules via dual polytope projection. In: Advances in neural information processing systems. pp. 1070--1078 (2013)

\bibitem{wu2024auto}
Wu, X., Gao, S., Zhang, Z., Li, Z., Bao, R., Zhang, Y., Wang, X., Huang, H.: Auto-train-once: Controller network guided automatic network pruning from scratch. In: Proceedings of the IEEE/CVF Conference on Computer Vision and Pattern Recognition. pp. 16163--16173 (2024)

\bibitem{xiang2016screening}
Xiang, Z.J., Wang, Y., Ramadge, P.J.: Screening tests for lasso problems. IEEE transactions on pattern analysis and machine intelligence  \textbf{39}(5),  1008--1027 (2016)

\bibitem{xiao2014proximal}
Xiao, L., Zhang, T.: A proximal stochastic gradient method with progressive variance reduction. SIAM Journal on Optimization  \textbf{24}(4),  2057--2075 (2014)

\bibitem{yuan2006model}
Yuan, M., Lin, Y.: Model selection and estimation in regression with grouped variables. Journal of the Royal Statistical Society: Series B (Statistical Methodology)  \textbf{68}(1),  49--67 (2006)

\bibitem{zhang2018learning}
Zhang, D., Wang, H., Figueiredo, M., Balzano, L.: Learning to share: Simultaneous parameter tying and sparsification in deep learning  (2018)

\bibitem{zhao2009composite}
Zhao, P., Rocha, G., Yu, B.: The composite absolute penalties family for grouped and hierarchical variable selection. The Annals of Statistics  \textbf{37}(6A),  3468--3497 (2009)

\bibitem{zhong2012efficient}
Zhong, L.W., Kwok, J.T.: Efficient sparse modeling with automatic feature grouping. IEEE transactions on neural networks and learning systems  \textbf{23}(9),  1436--1447 (2012)

\end{thebibliography}

\end{document}